\documentclass[11pt]{article} % use larger type; default would be 10pt

\usepackage[utf8]{inputenc} % set input encoding (not needed with XeLaTeX)
\usepackage{geometry} % to change the page dimensions
\usepackage{fancyhdr} 
\usepackage{amsthm}

\usepackage[colorlinks]{hyperref} % remove this in the submission

\usepackage[utf8]{inputenc} 
\usepackage[T1]{fontenc}
\usepackage{url}
\usepackage{ifthen}
\usepackage{cite}
\usepackage[cmex10]{amsmath} % Use the [cmex10] option to ensure complicance
\usepackage{graphicx}
\usepackage{amssymb}
\usepackage{xcolor}
\newtheorem{theorem}{Theorem}

\newtheorem{lemma}{Lemma}
\newtheorem{remark}{Remark}
\newtheorem{corollary}{Corollary}
\newtheorem{definition}{Definition}
\newtheorem{assumption}{Assumption}

\def \E {\mathop{{}\mathbb{E}}}

\def \scriptZ {\mathcal{Z}}
\def \scriptW {\mathcal{W}}
\def \real {\mathbb{R}}
\def \mprob {\mathbb{P}}

\begin{document}

\begin{center}

	{\bf{\LARGE{Generalization Error Bounds for Noisy, Iterative Algorithms}}}

	\vspace*{.25in}

	\begin{tabular}{ccccc}
		{\large{Ankit Pensia}$^*$} & \hspace*{.75in} & {\large{Varun Jog$^\dagger$}} & \hspace*{.75in} & {\large{Po-Ling Loh$^{*\dagger}$}}\\
		{\large{\texttt{ankitp@cs.wisc.edu}}} & \hspace*{.75in} & {\large{\texttt{vjog@ece.wisc.edu}}} & \hspace*{.75in} & {\large{\texttt{loh@ece.wisc.edu}}}
			\end{tabular}
\begin{center}
Departments of Computer Science$^*$  and Electrical \& Computer Engineering$^\dagger$\\
University of Wisconsin - Madison\\
1415 Engineering Drive\\
Madison, WI 53706
\end{center}

	\vspace*{.2in}

	January 2018

	\vspace*{.2in}

\end{center}

\begin{abstract}
In statistical learning theory, generalization error is used to quantify the degree to which a supervised machine learning algorithm may overfit to training data. Recent work [Xu and Raginsky (2017)] has established a bound on the generalization error of empirical risk minimization based on the mutual information $I(S;W)$ between the algorithm input $S$ and the algorithm output $W$, when the loss function is sub-Gaussian. We leverage these results to derive generalization error bounds for a broad class of iterative algorithms that are characterized by bounded, noisy updates with Markovian structure. Our bounds are very general and are applicable to numerous settings of interest, including stochastic gradient Langevin dynamics (SGLD) and variants of the stochastic gradient Hamiltonian Monte Carlo (SGHMC) algorithm. Furthermore, our error bounds hold for any output function computed over the path of iterates, including the last iterate of the algorithm or the average of subsets of iterates, and also allow for non-uniform sampling of data in successive updates of the algorithm.
\end{abstract}

\section{Introduction}

Many popular machine learning applications may be cast in the framework of empirical risk minimization (ERM)~\cite{Vap98, ShaBen14}. This risk is defined as the expected value of an appropriate loss function, where the expectation is taken over a population. Rather than minimizing the risk directly, ERM proceeds by minimizing the empirical average of the loss function evaluated on the finite sample of data points contained in the training set~\cite{ShaEtal10}. In addition to obtaining a computationally efficient, near-optimal solution to the ERM problem, it is therefore necessary to quantify how much the empirical risk deviates from the true risk of the loss function, which in turn dictates the closeness of the ERM estimate to the underlying parameter of the data-generating distribution.

In this paper, we focus on a family of iterative ERM algorithms, and derive generalization error bounds for the parameter estimates obtained from such algorithms. A unifying characteristic of the iterative algorithms considered in our paper is that each successive update includes the addition of noise, which prevents the learning algorithm from overfitting to the training data. Furthermore, the iterates of the algorithm are related via a Markov structure, and the difference between successive updates (disregarding the noise term) is assumed to be bounded. One popular learning algorithm of this nature is stochastic gradient Langevin dynamics (SGLD)---which may be viewed as a version of stochastic gradient descent (SGD) that injects Gaussian noise at each iteration---applied to a loss function with bounded gradients. Our approach leverages recent results that bound the generalization error using the mutual information between the input data set and the output parameter estimates~\cite{RusZou16, XuRag17}. Importantly, this technique allows us to apply the chain rule of mutual information and leads to a simple analysis that extends to estimates that are obtained as an arbitrary function of the iterates of the algorithm. The sampling strategy may also be data-dependent and allowed to vary over time, but should be agnostic to the parameters. 

Generalization properties of SGD have recently been derived using a different approach involving algorithmic stability~\cite{HarEtal16, Lon16}. The main idea is that learning algorithms that change by a small bounded amount with the addition or removal of a single data point must also generalize fairly well~\cite{BouEli02, EliEtal05, MukEtal06}. However, the arguments employed to show that SGD is a stable algorithm crucially rely on the fact that the updates are obtained using bounded gradient steps. Mou et al.~\cite{MouEtal17} provide generalization error bounds for SGLD by relating stability to the squared Hellinger distance, and bounding the latter quantity. Although their generalization error bounds are tighter than ours in certain cases, our approach based on a purely information-theoretic notion of stability (i.e., mutual information) allows us to consider much more general classes of updates and final outputs, including averages of iterates; furthermore, the algorithms analyzed in our framework may perform iterative updates with respect to a non-uniform sampling scheme on the training data set.

The remainder of the paper is organized as follows: In Section~\ref{SecSetting}, we introduce the notation and assumptions to be used in our paper. In Section~\ref{SecMain}, we present the main result bounding the mutual information between inputs and outputs for our class of iterative learning algorithms, and derive generalization error bounds in expectation and with high probability. In Section~\ref{SecExamples}, we provide illustrative examples bounding the generalization error of various noisy algorithms. We conclude with a discussion of related open problems. Detailed proofs of supporting lemmas are contained in the Appendix.

%For detailed proofs of the supporting lemmas included in the text, we refer the reader to the longer version
%\footnote{\url{http://pages.cs.wisc.edu/\textasciitilde ankitp/research/gen\_error\_noisy\_algorithms.pdf}}
%of the paper.

%%%%

\section{Problem setting}
\label{SecSetting}

We begin by fixing some notation to be used in the paper, and then introduce the class of learning algorithms we will study. We write $\|\cdot\|_2$ to denote the Euclidean norm of a vector. For a random variable $X$ drawn from a distribution $\mu$, we use $\E_{X \sim \mu}$ to denote the expectation taken with respect to $X$. We use $\mu^{\otimes n}$ to denote the product distribution constructed from $n$ independent copies of $\mu$. We write $I_d$ to denote the $d$-dimensional identity matrix. 

%%%%

\subsection{Preliminaries}

Suppose we have an \emph{instance space} $\scriptZ$ and a \emph{hypothesis space} $\scriptW$ containing the possible parameters of a data-generating distribution. We are given a training data set $S = \{z_1, z_2, \dots, z_n\}$ drawn from $\scriptZ$, where $z_i \stackrel{i.i.d.}{\sim} \mu$. Let $\ell: \scriptW \times \scriptZ \rightarrow \real$ be a fixed loss function. We wish to find a parameter $w \in \mathbb{R}^d$ that minimizes the \emph{risk} $L_\mu$, defined by
\begin{equation*}
L_{\mu}(w) :=  \E_{Z \sim \mu} [\ell(w, Z)].
\end{equation*}
For example, the setting of linear regression corresponds to case where $\scriptZ = \real^d \times \real$ and $z_i = (x_i, y_i)$, where each $x_i \in \real^d$ is a covariate and $y_i \in \real$ is the associated response. Furthermore, using the loss function $\ell(w, z) = (y - x^T w)^2$ corresponds to a least squares fit.

In the framework of ERM, we are interested in the \emph{empirical risk}, defined to be the empirical average of the loss function computed with respect to the training data:
\begin{align*}
L_S(w) := \frac{1}{n} \sum_{i=1}^n \ell(w, z_i).
\end{align*}
A learning algorithm may be viewed as a channel that takes the data set $S$ as an input and outputs an estimate $W$ from a distribution $\mprob_{W|S}$. In canonical ERM, where $W$ is simply the minimizer of $L_S(w)$ in $\scriptW$, the conditional distribution $\mprob_{W|S}$ is degenerate; however, when a stochastic algorithm is employed to minimize $L_S(w)$, the distribution $\mprob_{W|S}$ may be non-degenerate (and convergent to a delta mass at the true data-generating distribution if the algorithm is consistent).

For an estimation algorithm characterized by the distribution $\mprob_{W|S}$, we define the \emph{generalization error} to be the expected difference between the empirical risk and the actual risk, where the expectation is taken with respect to both the data set $S \sim \mu^{\otimes n}$ and the randomness of the algorithm:
\begin{equation*}
\text{gen}(\mu, \mprob_{W|S}) := \E_{S \sim \mu^{\otimes n}, W \sim \mprob_{W|S}}[L_{\mu}(W) - L_S(W)].
\end{equation*}
%Note in particular that the expected risk of the algorithm can then be bounded using the triangle inequality as
%\begin{equation*}
%\E_{S \sim \mu^{\otimes n}, W \sim P_{W|S}} [L_{\mu}(W)] \le |\text{gen}(\mu, P_{W|S})| + |L_S(W) - L_S(w^*_S)|
%\end{equation*}
The \emph{excess risk}, defined as the difference between the expected loss incurred by the algorithm and the true minimum of the risk, may be decomposed as follows:
\begin{align*}
\E_{S \sim \mu^{\otimes n}, W \sim \mprob_{W|S}} [L_{\mu}(W)] - L_{\mu}(w^*) = \text{gen}(\mu, \mprob_{W|S}) 
+ \left(\E[L_S(W)] - L_{\mu}(w^*)\right),
\end{align*}
where $w^* := \arg\min_{w \in \scriptW} \E_{Z \sim \mu}[\ell(w, Z)]$. Furthermore, it may be shown (cf.\ Lemma 5.1 of Hardt et al.~\cite{HarEtal16}) that $\E[L_S(w^*_S)] \le L_{\mu}(w^*)$, where $w^*_S := \arg\min_{w \in \scriptW} L_S(w)$ is the true empirical risk minimizer. Hence, we have the bound
\begin{equation}
\label{EqnGenOpt}
\E_{S \sim \mu^{\otimes n}, W \sim \mprob_{W|S}} [L_{\mu}(W)] - L_{\mu}(w^*) \le |\text{gen}(\mu, \mprob_{W|S})| + \epsilon_{\text{opt}}^W,
\end{equation}
where
\begin{equation*}
\epsilon_{\text{opt}}^W := \left|\E[L_S(W)] - \E[L_S(w^*_S)]\right|
\end{equation*}
denotes the \emph{optimization error} incurred by the algorithm in minimizing the empirical risk.

%%%%%

\subsection{Generalization error bounds}

The idea of bounding generalization error by the mutual information $I(W;S)$ between the input and output of an ERM algorithm was first proposed by Russo and Zou~\cite{RusZou16} and further investigated by Xu and Raginsky~\cite{XuRag17}. We now describe their results, which will be instrumental in our work. Recall the following definition:

\begin{definition}
A random variable $X$ is \emph{$R$-sub-Gaussian} if the following inequality holds:
\begin{align*}
\E[\exp(\lambda(X - \E X))] \leq \exp\left(\frac{\lambda^2 R^2}{2}\right), \qquad \forall \lambda \in \real.
\end{align*}
\end{definition}
We will assume that the loss function is uniformly sub-Gaussian in the second argument over the space $\scriptW$:
\begin{assumption}
Suppose $\ell(w, Z)$ is $R$-sub-Gaussian with respect to $Z \sim \mu$, for every $w \in \scriptW$.
\end{assumption}
In particular, if $\mu$ is Gaussian and $\ell(w, Z)$ is Lipschitz, then $\ell(w, Z)$ is known to be sub-Gaussian~\cite{BouEtal13}. Under this assumption, we have the following result:
\begin{lemma}[Theorem 1 of Xu and Raginsky~\cite{XuRag17}]
\label{LemXuRag}
Under Assumption 1, the following bound holds:
\begin{align}
    |\text{gen}(\mu, \mprob_{W|S})| &\leq \sqrt{\frac{2R^2}{n}I(S;W)}.
    \label{eq:miGeneralization}
\end{align}
\end{lemma}
In other words, the generalization error is controlled by the mutual information, supporting the intuition that an algorithm without heavy dependence on the data will avoid overfitting.

%%%%

\subsection{Class of learning algorithms}

% \textcolor{blue}{Mention that most learning algorithms perform some kind of iterative updates to optimize the empirical risk. For example gradient descent, SGD, SGLD, etc. Then say that for this reason, we focus our attention on such iterative algorithms.} 
%Most of the learning algorithms are iterative in nature and try to optimize the empirical risk, for example, gradient descent, SGD, SGLD, etc. For this reason, we focus our attention on such iterative algorithms.

We now define the types of ERM algorithms to be studied in our paper. We will focus on algorithms that proceed by iteratively updating a parameter estimate based on samples drawn from the data set $S$. Our theory is applicable to algorithms that make noisy, bounded updates on each step, such as the SGLD algorithm applied to a loss function with uniformly bounded gradients.

Denote the parameter vector at iterate $t$ by $W_t \in \mathbb R^d$, and let $W_0 \in \scriptW$ denote an arbitrary initialization. At each iteration $t \ge 1$, we sample a data point $Z_t \subseteq S$ and compute a direction $F(W_{t-1}, Z_t) \in \real^d$. We then scale the direction vector by a stepsize $\eta_t$ and perturb it by isotropic Gaussian noise $\xi_t \sim N(0, \sigma_t^2 I_d)$, to obtain the overall update
\begin{align}
     W_{t} = g(W_{t-1}) - \eta_t F(W_{t-1}, Z_t) + \xi_t, \qquad \forall t \ge 1,
     \label{eq:updateEqn} 
 \end{align} 
where $g: \real^d \rightarrow \real^d$ is a deterministic function. An important special case is when $g$ is the identity function and $F$ is a (clipped) gradient of the loss function: $F(w,z) = \nabla_w \ell(w,z)$. This leads to the familiar updates of the SGLD algorithm~\cite{WelTeh11}. For examples of settings where $g$ is a non-identity function, see the discussion of momentum and accelerated gradient methods in Section~\ref{SecExamples} below.

\begin{remark}
\label{RemNoise}
Our analysis does not actually require the noise vectors $\{\xi_t\}$ to be Gaussian, as long as they are drawn from a continuous distribution. The proofs would continue to hold with minimal modification, but would lead to sub-optimal bounds---indeed, a careful examination of our proofs shows that Gaussian noise produces the tightest bounds, because Gaussian noise has the maximum entropy for a fixed variance. Our results also generalize to settings where $Z_t$ may be a collection of data points drawn from $S$ and $F$ is computed with respect to all the data points (e.g., a mini-batched version of SGD), provided the sampling strategy satisfies the Markov structure imposed by Assumption~\ref{AssSamp} below.
\end{remark}

For $t \ge 0$, let $W^{(t)} := (W_1, \hdots, W_t)$ and $Z^{(t)} := (Z_1, \hdots, Z_t)$. We impose the following assumptions on $g$, $F$, and the dependency structure between the $W$'s and $Z$'s:
\begin{assumption}
\label{AssBound}
The updates are bounded; i.e., $\sup_{w \in \scriptW, z \in \scriptZ} \|F(w,z)\|_2 \leq L$, for some $L > 0$.
\end{assumption}

\begin{assumption}
\label{AssSamp}
The sampling strategy is agnostic to the previous iterates of the parameter vectors: 
\begin{align}
\label{eq:samplingConditionUpdated}
\mprob(Z_{t+1} \mid Z^{(t)}, W^{(t)},S) = \mathbb \mprob(Z_{t+1} | Z^{(t)},S).                                    
\end{align}
\end{assumption}

% \textcolor{blue}{A.2 is confusing. When you say sampling strategy, the reader expects to see $\mathbb P(X_{t+1}| X^{(t)}, W^{(t)}) = \mathbb P(X_{t+1} | X^{(t)})$. I think the current version of A.2 is a consequence of this, right?}
% \textcolor{red}{ The current version of A.2 implies $\mathbb P(X_{t+1}| X^{(t)}, W^{(t)}) = \mathbb P(X_{t+1} | X^{(t)})$. Basically, it applies all the small lemmas that we prove. I agree that it is not `readable'.
% $\mathbb P(X_{t+1}| X^{(t)}, W^{(t)}) = \mathbb P(X_{t+1} | X^{(t)})$ and $\mathbb P(W_{t+1}| W^{(t)}, X^{(t+1)}) = \mathbb P(W_{t+1} | W_t, X_{t+1})$ implies current A.2.
% }.

Note that the update equation~\eqref{eq:updateEqn} implies that $\mprob(W_{t+1}| W^{(t)}, Z^{(t+1)},S) = \mprob(W_{t+1} | W_t, Z_{t+1})$, which combined with the sampling strategy~\eqref{eq:samplingConditionUpdated} implies the following conditional independence relation:
\begin{align} 
\label{eq:samplingCondition}
\mprob\left(W_{t+1} | W^{(t)}, Z^{(T)}, S\right) &= \mprob\left(W_{t+1} | W_t, Z_{t+1} \right),
\end{align}
where $T$ denotes the final iterate. We may represent the dependence structure defined by our class of algorithms in the form of a graphical model (see Figure~\ref{fig:markov} in the Appendix).

\begin{remark}
Importantly, we do not impose any further restrictions on the form of the updates or the sampling strategy; in particular, $Z_t$ need not be drawn uniformly from the data set $S$, and may even depend on past iterates $\{Z_s\}_{s<t}$, as in the case of sampling without replacement. Some examples of iterative algorithms where the probability of sampling a data point $z_i$ depends on the value of $z_i$ may be found in Zhao and Zhang~\cite{ZhaZha15} or Needell et al.~\cite{NeeEtal16}---such sampling strategies are also covered by our theory. However, note that $Z_t$ must be independent of the parameter iterates $\{W_s\}_{s<t}$, since if edges exist between $W_t$ and any $Z_s$ such that $s > t$, equation~\eqref{eq:samplingCondition} will not hold. Intuitively, if the sampled data point adapts to current iterates of the parameter vector $W_t$, the algorithm may be prone to over-fitting and may not generalize.

Finally, note that our assumptions do not require the loss function $\ell$ to satisfy conditions such as convexity. In fact, the way we have defined the updates~\eqref{eq:updateEqn} does not require $F$ to be related to $\ell$ in any way. On the other hand, if $F$ is essentially a gradient of $\ell$, as is often the case, Assumption~\ref{AssBound} will be satisfied as long as $\ell$ is Lipschitz in its first argument.
\end{remark}

The output of our estimation algorithm is defined to be an arbitrary function of the $T$ iterates: $W = f(W^{(T)})$. Some common examples appearing in the ERM literature include (i) the mean: $f (W^{(T)}) = \frac{1}{T} \sum_{t=1}^T{W_t}$; (ii) the last iterate: $f (W^{(T)}) = W_T$; or (iii) suffix averaging, and variants thereof~\cite{RakEtal12, ShaZha13}.

\section{Main results}
\label{SecMain}

We now derive an upper bound on $I(S;W)$ for the class of iterative algorithms described in Section~\ref{SecSetting}, from which we obtain bounds on the generalization error.

%%%%%%

\subsection{Bound on mutual information}

\begin{theorem}
\label{theorem:bmi}
The mutual information satisfies the bound
\begin{align*}
I(S;W) &\leq \sum_{t=1}^T\frac{d}{2}\log\left( 1 + \frac{\eta_t^2L^2}{d\sigma_t^2}\right).
% %
% & \leq \sum_{t=1}^T\frac{\eta_t^2L^2}{2\sigma_t^2   }\frac{1}{\left(\sqrt{ 1 + \frac{\eta_t^2L^2}{d\sigma_t^2} }  \right)}.
\end{align*}
% \textcolor{blue}{Any reason why the second bound is stated with the extra square root term? You can get rid of the whole second bound too.}
\end{theorem}
% \textcolor{red}
% {It was just a stronger bound than simply x. I agree that we can remove the second bound altogether.}
\begin{proof}
\begin{align}
I(S;W)  &= I(S; f(W^{(T)})) \leq I(S; W^{(T)}) \nonumber\\
        &\leq I(Z^{(T)}; W^{(T)}) \nonumber\\
        &= I(Z^{(T)};W_1) +  I(Z^{(T)};W_2 | W_1 ) \nonumber \\
        &\quad + I(Z^{(T)};W_3 | W_1, W_2 ) + \nonumber  \cdots +  I(Z^{(T)};W_T | W^{(T-1)} ) \label{eq:miBar1}  
\end{align}
where the inequality follows from Lemma \ref{lemma:infoSubsetMarkov} and the last equality comes from the chain rule of mutual information. 

For all $t$,
\begin{align}
& I(Z^{(T)};W_t | W^{(t-1)}) \nonumber \\
& \quad = h(W_t | W^{(t-1)}) - h(W_t | W^{(t-1)},Z^{(T)})  \nonumber\\
                            & \quad \stackrel{(a)}{=} h(W_t | W_{t-1}) - h(W_t | W_{t-1},Z_t ) \nonumber\\
                            & \quad = I(W_t;Z_t | W_{t-1}) \nonumber\\ 
                        & \quad \stackrel{(b)}{\leq} \frac{d}{2}\log\left( 1 + \frac{\eta_t^2L^2}{d\sigma_t^2}\right),
\end{align}
where equality $(a)$ follows from Lemma \ref{lemma:dataSamplingMarkov} and Lemma \ref{lemma:entropyMarkov} in the Appendix, whereas inequality $(b)$ follows from Lemma \ref{lemma:infoBound}.

Therefore, Eq. \eqref{eq:miBar1} gives
\begin{align}
I(S;W^{(T)}) &\leq \sum_{t=1}^T\frac{d}{2}\log\left( 1 + \frac{\eta_t^2L^2}{d\sigma_t^2}\right).
% \\
% I(S;W^{(T)}) &\leq B_T
\end{align}
% To obtain the second bound, we use the fact that $\log(1 + x) \leq \frac{x}{\sqrt{1+x}}$, $\forall x > 0$.
We may obtain bounds without a log term by using the fact that $\log(1 + x) \leq \frac{x}{\sqrt{1+x}} < x$, $\forall x > 0$.  
\end{proof}

%%%%%

\subsection{Consequences} % (fold)
\label{sec:bound_on_generalization_error}

%Assume that $l(w,x)$ is $R$-sub-Gaussian for any $w \in W$. Suppose we run the \texttt{PAIUSNB} algorithm for $T$ time steps,  with step sizes $\eta_t$ and noise covariances $\sigma_t^2I_d$. By Theorem 1, we have that $I(S;W) \leq \sum_{t=1}^T\frac{\eta_t^2L^2}{ 2\sigma_t^2} $.

%\subsection{Bound in Expectation}

We now use this bound on mutual information from Theorem~\ref{theorem:bmi} to derive bounds on the generalization error, first in expectation and then with high probability. The first bound follows directly from Theorem~\ref{theorem:bmi} and Lemma~\ref{LemXuRag}:

\begin{corollary} [Bound in expectation]
\label{CorExpect}
The generalization error of our class of iterative algorithms is bounded by
\begin{align}
|\text{gen}(\mu,P_{W|S})| \leq \sqrt{\frac{R^2}{n}\sum_{t=1}^T\frac{\eta_t^2L^2}{\sigma_t^2}}.
\label{eq:expectedBoundOnGeneralization}
\end{align}
\end{corollary}

%\subsection{High-probability bound}

Similarly, Theorem 3 in Xu and Raginsky~\cite{XuRag17} implies a generalization error bound that holds with high probability:

\begin{corollary}\label{CorHP} [High-probability bound]
Let $I(S;W) \leq \epsilon$. Then by Theorem \ref{theorem:bmi}, $\epsilon$ can be equal to $\sum_{t=1}^T\frac{d}{2}\log\left( 1 + \frac{\eta_t^2L^2}{d\sigma_t^2}\right)$.
For any $\alpha > 0$ and $0 < \beta \leq 1$, if $n > \frac{8 R^2 }{\alpha^2}\left( \frac{\epsilon}{\beta} + \log(\frac{2}{\beta})\right)$, we have
\begin{align} 
\mprob_{S,W} \left(|L_{\mu}(W) - L_S(W) | > \alpha \right) \leq \beta,
\label{eq:pacResults}
\end{align}
where the probability is with respect to $S \sim \mu^{\otimes n}$ and $W$.
\end{corollary}

% A corollary (cf. Corollary 1 of Xu and Raginsky~\cite{XuRag17} ) of above is that if $I(S;W) \leq (\sqrt{n}- 1)\beta \log(\frac{2}{\beta}) $, then $n = \frac{64 R^4 }{\alpha^4} (\log(\frac{2}{\beta}))^2$ satisfies Eq. \eqref{eq:pacResults}. 
% This form of the result implies that as the data grows in the size, we can extract more information from the data. 
% (\textcolor{red}{Not sure what this last comment is referring to? Seems to be repeated in the next section.})
% \textcolor{blue}{the constraints on $n,\epsilon, \alpha,\beta$ to satisfy inequality of eq. 10 is tricky to work with as all the qualities are inter-related. This corollary simplifies the constraints.}
%%%%

\section{Examples}
\label{SecExamples}

We now apply the corollaries in Section~\ref{sec:bound_on_generalization_error} to obtain generalization error bounds for various algorithms.

\subsection{SGLD}

As mentioned earlier, sampling the data points uniformly and setting $g(w) = w$ and $F(w,z) = \nabla_w \ell(w,z)$ corresponds to the SGLD algorithm. Common experimental practices for SGLD are as follows~\cite{WelTeh11}:
\begin{enumerate}
    \item the noise variance is set to be $\sigma_t = \sqrt{\eta_t}$,
    \item the algorithm is run for $K$ epochs; i.e., $T = nK$,
    \item for a constant $c > 0$, the stepsizes are $\eta_t = \frac{c}{t}$.
\end{enumerate}

\subsubsection*{High-probability bounds}

For a given choice of $\{\beta, \alpha\}$, taking $n \geq \frac{64 R^4 }{\alpha^4} \left(\log(\frac{2}{\beta})\right)^2$ ensures inequality~\eqref{eq:pacResults}, provided that we run $ K \leq \frac{1}{ne}\left(\frac{2 }{\beta} ^ {\frac{2(\sqrt{ n}-1 ) \beta }{cL^2} }\right) $ epochs. For more details, see Lemma~\ref{lemma:epochBound} in Appendix~\ref{AppOptSGLD}.

\subsubsection*{Bounds in expectation}

Using the identity \mbox{$\sum_{t=1}^T\frac{1}{t} \leq \log(T)+1$}, we obtain the following bound:
\begin{align*}
    |\text{gen}(\mu,\mprob_{W|S})| &\leq \frac{RL}{\sqrt{n}}\sqrt{\sum_{t=1}^T\eta_t} \leq \frac{RL}{\sqrt{n}}\sqrt{c\log T+c}.
\end{align*}
Note that Mou et al.~\cite{MouEtal17} achieve a tighter bound on generalization error of the order $\mathcal{O}\left(\frac{1}{n}\right)$, but their bound is only applicable to the last iterate $W_T$ of SGLD and a uniform sampling strategy. 

% (\textcolor{red}{Include comparison with Chinese authors.})

\subsubsection*{Convex risk minimization}

If the loss function $\ell(w,z)$ is convex in its first argument for every $w$, we may also bound the excess risk of the learning algorithm. 
%Let $w^*_S$ denote the $ \arg\inf_w L_S(w)$ and  Then, $\epsilon_{\text{opt}}^W = \E[L(W) - L_S(w_S^*)]$ denotes the optimization error of the learning algorithm where expectation is taken over the learning algorithm and $S$.
%Following \cite[Lemma 5.1]{HarEtal16}, we can write the true risk of the function as 
%\begin{align}
%    \E[L_S[W]] &\leq  L(w^*) +  \epsilon_{\text{opt}}^W + \text{gen}(\mu,P_{W|S})  
%\end{align}
Recall the bound~\eqref{EqnGenOpt} and the definition of the optimization error. It may be shown (cf.\ Lemma~\ref{claim:ermsgld} in Appendix~\ref{AppOptSGLD}) that when $(\eta_t, \sigma_t) = (\eta, \sigma)$ and $W = \frac{1}{T}\sum_{t=1}^{T} W_t$, the optimization error of SGLD satisfies
\begin{align}
\epsilon_{\text{opt}}^{W} \leq \frac{G^2}{2\eta T} +  \frac{\eta}{2} L^2 + \frac{d\sigma^2}{2\eta},
\label{eq:optimizationError}
\end{align}
where $G = \sup_{w,S} \|w_{0} - w^*_S\|_2$. By inequalities~\eqref{EqnGenOpt}, \eqref{eq:expectedBoundOnGeneralization}, and~\eqref{eq:optimizationError}, we then have
\begin{align*}
\E[L_S[W]]   &\leq L(w^*) + \frac{G^2}{2\eta T} +  \frac{\eta}{2} L^2 + \frac{d\sigma^2}{2\eta} + \frac{R\sqrt{T}}{\sqrt{n}}\frac{\eta L}{ \sigma }.
%\label{eq:RMError}
\end{align*}
Setting $\sigma = \frac{G}{\sqrt{dT}} $ and  $\eta = \sqrt{\frac{G^2}{TL(\frac{L}{2}+ \frac{R\sqrt{d}T}{\sqrt{n}G})} }$, we obtain
\begin{align*}
    \E[L_S[W]] - L(w^*) &\leq 2GL\sqrt{  \frac{1}{2T} + \frac{\sqrt{d}}{\sqrt{n}}\frac{R }{ GL } }.
%\label{eq:sgldresult}
\end{align*}

% \textcolor{red}{Can we bound $G = \text{sup }_w|| w - w^*_S||$, i.e., can we use this algorithm for constrained optimization? If we just restrict w to lie in a ball, then the surface of the ball would have positive probability but zero volume, so its entropy won't be defined. }
%\section{Generalization bounds for Accelerated Gradient Descent Algorithms}

\subsection{Perturbed SGD}
%\textcolor{blue}{This does not look correct to me. Noise on the sphere (i.e. the shell of a ball) will have $-\infty$ entropy making the bound useless. This is a really annoying property of entropy and I wish there were a way around it...}

Due to the requirement that an independent noise term $\xi_t$ is present in each update, our results on generalization error may not be applied to SGD. On the other hand, our framework does apply to \emph{noisy} versions of SGD, which have recently drawn interest in the optimization literature due to their ability to escape saddle points efficiently~\cite{GeEtal15, JinEtal17}. For a stepsize parameter $\eta > 0$, updates of the perturbed SGD algorithm take the following form~\cite{GeEtal15}:
\begin{equation}
\label{EqnNoisySGD}
W_t = W_{t-1} - \eta \left(\nabla_w \ell(W_{t-1}, Z_t) + \xi_t\right),
\end{equation}
where $\xi_t$ are i.i.d.\ noise terms sampled uniformly from the unit sphere. Hence, noise is added to each gradient. Unfortunately, our techniques cannot be applied to this exact setting because $\xi_t$ has a degenerate distribution concentrated on the sphere. For large enough $d$, choosing $\xi_t$ on the unit sphere is almost equivalent to choosing it inside the unit ball. If $\xi_t$ is chosen uniformly in the unit ball (cf.\ the perturbed SGD formulation in Jin et al.~\cite{JinEtal17}), our methods yield the following bound:
\begin{equation}\label{EqnNoisySGDBound}
I(S; W) \leq T d\log (1+L). 
\end{equation}
This is because $\|W_t - W_{t-1}\|_2 \le \eta(L+1)$, so we may bound $h(W_t | W_{t-1})$ by the entropy of the uniform distribution on the $d$-dimensional ball of radius $\eta(L+1)$. Also, $h(W_t| W_{t-1}, Z_t)$ is simply the entropy of the uniform distribution on the $d$-dimensional ball of radius $\eta$. This shows that $I(W_{t}; Z_t | W_{t-1}) \leq d\log (1+L)$, so
\begin{equation*}
I(W; S) \leq I(W^{(T)}; S) \leq I(W^{(T)}; Z^{(T)}) \leq Td\log(1+L).
\end{equation*}
%Another variant~\cite{JinEtal17} uses the updates
%\begin{equation}
%\label{EqnPerturbedSGD}
%W_t = (W_{t-1} + \xi_t) - \eta \nabla_w \ell(W_{t-1} + \xi_t, Z_t),
%\end{equation}
%where the previous update is first contaminated by noise and then a gradient step is taken.

%Note that in the case of updates~\eqref{EqnNoisySGD}, we may simply treat the additive noise term at each iterate as $\eta \xi_t$. As discussed in Remark~\ref{RemNoise}, similar bounds on mutual information apply when the noise vectors are sampled uniformly from a sphere rather than a Gaussian distribution (\textcolor{red}{what is the new bound?}). For the algorithm described by the updates~\eqref{EqnPerturbedSGD}, note that the perturbed vectors $W_t' = W_t + \xi_{t+1}$ satisfy the recursion
%\begin{equation*}
%W_t' = W_{t-1}' - \eta \nabla_w \ell(W_{t-1}', Z_t) + \xi_{t+1}.
%\end{equation*}
%This is exactly the same equation as the updates of SGLD, except the noise error vectors are again sampled from uniformly a sphere rather than a Gaussian distribution. This leads to the generalization error bounds (\textcolor{red}{fill in}).
%%%%%

\subsection{Noisy momentum}

In this section, we show how we can develop bounds for momentum-like algorithms in addition to SGLD. We consider an algorithm similar to the SGHMC algorithm \cite{CheEtal14}. Every iteration $t$ involves an extra parameter vector $V_t$, which represents the ``velocity'' of $W_t$. We analyze a modified SGHMC algorithm, where we add the (independent and Gaussian) noise $\xi'_{t}$ to the velocity, as well. This leads to the update equations
\begin{align}
\begin{aligned}
    V_{t} &= \gamma_t V_{t-1} + \eta_t \nabla_w \ell(W_{t-1},Z_{t}) + \xi'_{t}, \\
    W_{t} 
    % &= W_t - V_{t+1} + Z''_{t+1} + Z'_{t+1}\\
            % &= W_t -  (\gamma_t V_t + \eta_t \nabla l(W_t,X_t) + Z'_{t+1}) + Z''_{t+1} \\
            &= W_{t-1} - \gamma_t V_{t-1} - \eta_t \nabla_w \ell(W_{t-1},Z_t)  + \xi''_{t},
\end{aligned}
\label{eq:momentum}
\end{align}
or in matrix form,
\begin{align*}
 \begin{bmatrix} V_{t} \\  W_{t}\end{bmatrix}     &=   \begin{bmatrix}\gamma_t & 0 \\
-\gamma_t & 1
\end{bmatrix}
\begin{bmatrix} V_{t-1} \\  W_{t-1}\end{bmatrix} +  \eta_t\begin{bmatrix}  \nabla_w \ell(W_{t-1},Z_t) \\ - \nabla_w \ell(W_{t-1},Z_t) \end{bmatrix} + \begin{bmatrix} \xi'_{t} \\  \xi''_{t} \end{bmatrix}.
\end{align*}
Thus, we may recast the updates in the framework of our paper by treating $(V_t,W_t)$ as a single parameter vector in $\mathbb{R}^{2d}$, with
\begin{align*}
g(V_{t-1},W_{t-1}) & = \begin{bmatrix}\gamma_t & 0 \\
-\gamma_t & 1
\end{bmatrix}
\begin{bmatrix} V_{t-1} \\  W_{t-1}\end{bmatrix}, \quad \text{and} \\
F\big((V_{t-1},W_{t-1}),Z_t\big) & = \begin{bmatrix}  \nabla_w \ell(W_{t-1},Z_t) \\ -\nabla_w \ell(W_{t-1},Z_t) \end{bmatrix}.
\end{align*}
Note that if the gradients are upper-bounded by $L$, we have $\sup_{v, w \in \scriptW, z \in \scriptZ} \|F((v,w),z)\|_2 \leq \sqrt{2}L$.
% We just need to bound the entropy term $h(V_t,W_t | V_{t-1},W_{t-1})$, and we use the upper-bound on variance to bound it.
% \begin{align*}
% \begin{aligned}
% & \text{Var} \left( \begin{bmatrix} V_{t+1} \\ W_{t+1} \end{bmatrix} \bigg| \begin{bmatrix} V_{t} \\ W_{t} \end{bmatrix} \right) \\
%     % \text{Var} \left( \begin{bmatrix} V_{t+1} \\ W_{t+1} \end{bmatrix} \bigg| \begin{bmatrix} V_{t} \\ W_{t} \end{bmatrix}  =  \begin{bmatrix} v_{t} \\ w_{t} \end{bmatrix} \right)
%      & \qquad \leq \text{Var}\left(  \begin{bmatrix} \eta_t \nabla_w \ell(W_t,Z_t) \\ -\eta_t \nabla_w \ell(W_t,Z_t) \end{bmatrix} + \begin{bmatrix} \xi'_{t+1} \\  \xi''_{t+1} \end{bmatrix}  \right) \\
%     & \qquad \leq 2\eta_t^2 L^2 + 2 d\sigma_t^2.
% \end{aligned}
% \end{align*}
Using Theorem \ref{theorem:bmi}, we then arrive at the following bound:
\begin{align*}
I(S;W) &\leq \sum_{t=1}^T\frac{2d}{2}\log\left( 1 + \frac{\eta_t^22L^2}{2d\sigma_t^2}\right).
\end{align*}
Note that it is twice the bound on the mutual information appearing in Theorem~\ref{theorem:bmi}. We may then apply the results in Section~\ref{sec:bound_on_generalization_error} to obtain bounds on the generalization error:\begin{align*}
|\text{gen}(\mu,P_{W|S})| \leq \sqrt{\frac{2R^2}{n}\sum_{t=1}^T\frac{\eta_t^2L^2}{\sigma_t^2}}.
\label{eq:expectedBoundOnGeneralization}
\end{align*}

%%%%

\subsection{Accelerated gradient descent}

Finally, we consider a noisy version of the accelerated gradient descent method of Nesterov~\cite{Nes83}, where we again add independent noise to both the velocity and parameter vectors at each iteration. This leads to the update equations
\begin{align*}
\begin{aligned}
    V_{t+1} &= \gamma V_t +  \eta_t \nabla_w \ell(W_t - \gamma_t V_t,Z_t) + \xi'_{t+1}, \\
    W_{t+1} &= W_t -   V_{t+1} + \xi''_{t+1} + \xi'_{t+1}.
    % \\
% \implies \begin{bmatrix} V_{t+1} \\  W_{t+1}\end{bmatrix}     &=   \begin{bmatrix}\gamma_t & 0 \\
%                                                                          -\gamma_t & 1    \end{bmatrix}
%                                                           \begin{bmatrix} V_{t} \\  W_{t}\end{bmatrix} +  \begin{bmatrix} \eta_t \nabla l(W_t - \gamma_t V_t,X_t) \\ -\eta_t \nabla l(W_t - \gamma_t V_t,X_t) \end{bmatrix} + \begin{bmatrix} Z'_{t+1} \\ -Z'_{t+1} + Z''_{t+1} \end{bmatrix}
\end{aligned}
\end{align*}
We again consider $(V_t,W_t)$ as a single parameter vector in $\mathbb{R}^{2d}$.
Compared with the updates~\eqref{eq:momentum}, we see that the only difference is that the point where we take the gradient has changed. Therefore, we obtain the same bound on the $F\big((V_{t-1},W_{t-1}),Z_t\big)$ as in the case of noisy momentum:
$\sup_{v, w \in \scriptW, z \in \scriptZ} \|F((v,w),z)\|_2 \leq \sqrt{2}L$.
% \begin{align*}
%     \text{Var} \left( \begin{bmatrix} V_{t+1} \\ W_{t+1} \end{bmatrix} \bigg| \begin{bmatrix} V_{t} \\ W_{t} \end{bmatrix} \right)
%     % \text{Var} \left( \begin{bmatrix} V_{t+1} \\ W_{t+1} \end{bmatrix} \bigg| \begin{bmatrix} V_{t} \\ W_{t} \end{bmatrix}  =  \begin{bmatrix} v_{t} \\ w_{t} \end{bmatrix} \right)
%     &\leq 2\eta_t^2 L^2 + 2 d\sigma_t^2.
% \end{align*}
This leads to the same upper bound on the mutual information (and generalization error) as in the previous subsection. 

% \subsection{Bound on Information for Fast gradient methods}
% Therefore, the bound on information for both of the methods are - 
% \begin{align}
%     I(S;W) &\leq \sum_{t=1}^T\frac{2d}{2}\log\left( 1 + \frac{\eta_t^2L^2}{d\sigma_t^2}\right)
% \end{align}

% The style of references, equations, figures, tables, etc., should be
% the same as for the \emph{IEEE Transactions on Information
%   Theory}. The source file of this template paper contains many more
% instructions on how to format your paper. So, example code for
% different numbers of authors, for figures and tables, and references
% can be found (they are commented out).

%po-ling 
%(Do we want an example about variance reduction methods?)
%\textcolor{red}{We would have to add noise to all n gradients and this would increase the variance of the gradients, as opposed to the original intention of decreasing the variance. }

\section{Conclusion}

In this paper, we have demonstrated that mutual information is a very effective tool for bounding the generalization error of a large class of iterative ERM algorithms. The simplicity of our analysis is due to properties such as the data processing inequality and the chain rule of mutual information. However, entropy and mutual information also have certain shortcomings that limit the scope of our analysis, particularly concerning the sensitivity of entropy with respect to degenerate random variables. In some instances, mutual information-based bounds become very weak or even inapplicable. For example, if we were to analyze the SGD algorithm rather than SGLD, or add noise that is degenerate, such as the uniform distribution on a sphere~\cite{GeEtal15}, the mutual information $I(W;S)$ would be $+\infty$, leading to meaningless generalization error bounds. It would be interesting to develop information-theoretic strategies that could bound the generalization error for such algorithms, as well. Finally, note that we have only provided upper bounds for the generalization error---having a large $I(W;S)$ does not necessarily mean that an algorithm is overfitting, since our upper bound might be loose. Deriving lower bounds on the generalization error appears to be a challenging problem that could benefit from an information-theoretic approach, as well.

%%%%%%
%% Appendix:
%% If needed a single appendix is created by
%%
%\appendix
%%
%% If several appendices are needed, then the command
%%
% \appendices
%%
%% in combination with further \section-commands can be used.
%%%%%%

%\section*{Acknowledgment}

%%%%%%
%% To balance the columns at the last page of the paper use this
%% command:
%%
\enlargethispage{-1.2cm} 
%%
%% If the balancing should occur in the middle of the references, use
%% the following trigger:
%%
%\IEEEtriggeratref{3}
%%
%% which triggers a \newpage (i.e., new column) just before the given
%% reference number. Note that you need to adapt this if you modify
%% the paper.  The "triggered" command can be changed if desired:
%%
%\IEEEtriggercmd{\enlargethispage{-20cm}}
%%
%%%%%%

%%%%%%
%% References:
%% We recommend the usage of BibTeX:
%%
%\bibliographystyle{IEEEtran}
%\bibliography{definitions,bibliofile}
%%
%% where we here have assume the existence of the files
%% definitions.bib and bibliofile.bib.
%% BibTeX documentation can be obtained at:
%% http://www.ctan.org/tex-archive/biblio/bibtex/contrib/doc/
%%%%%%

%% Or you use manual references (pay attention to consistency and the
%% formatting style!)
\bibliographystyle{plain}
\bibliography{ref}

\begin{appendix}

\section{Proofs of supporting lemmas to Theorem~\ref{theorem:bmi}}

We now prove the lemmas employed in the proof of Theorem~\ref{theorem:bmi}.

\begin{figure*}[h]
\centering
\includegraphics[width=0.7\textwidth]{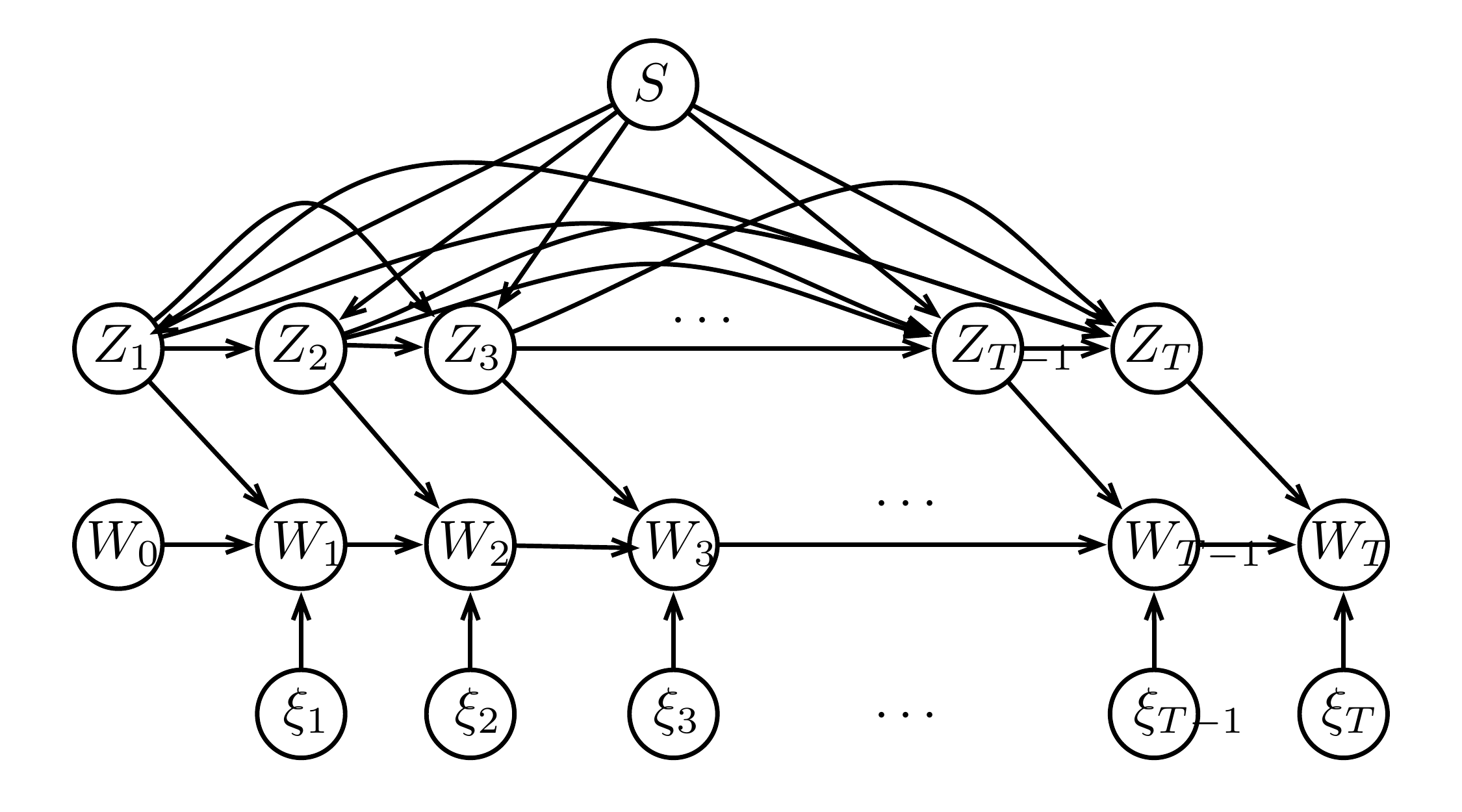}
\caption{Directed graphical model illustrating dependencies between data set $S$, samples $\{Z_t\}$, parameter iterates $\{W_t\}$, and noise vectors $\{\xi_t\}$.}
\label{fig:markov}
\end{figure*}

\begin{lemma}
$I(S;W) \leq I(Z^{(T)};W^{(T)})$.
\label{lemma:infoSubsetMarkov}
\end{lemma}
\begin{proof}
This follows from the Markov chain
\begin{equation*}
S \rightarrow Z^{(T)} \rightarrow W^{(T)}.
\end{equation*}
See equality~\eqref{eq:samplingCondition}. 
\end{proof}

\begin{lemma}
For all $t$, we have
\begin{equation*}
h(W_t | W^{(t-1)},Z^{(T)}) = h(W_t | W_{t-1},Z_t).
\end{equation*}
\label{lemma:dataSamplingMarkov}
\end{lemma}

\begin{proof}
This follows from the Markov chain
\begin{equation*}
(W^{(t-2)},Z^{(T) \backslash \{t\}}) \rightarrow (W_{t-1},Z_t) \rightarrow W_t,
\end{equation*}
where $Z^{(T) \backslash \{t\}} := (Z_1, \hdots, Z_{t-1},Z_{t+1},\hdots,Z_T)$.
See equality~\eqref{eq:samplingCondition}.
\end{proof}

\begin{lemma}
For all $t$, we have
\begin{equation*}
h(W_t | W^{(t-1)}) = h(W_t | W_{t-1}).
\end{equation*}
\label{lemma:entropyMarkov}
\end{lemma}
\begin{proof}
This follows from the Markov chain
\begin{equation*}
W^{(t-2)} \rightarrow W_{t-1} \rightarrow W_t.
\end{equation*}
See equality~\eqref{eq:samplingCondition}.
\end{proof}

\begin{lemma}
For all $t$, we have
\begin{equation*}
I(W_t;Z_t | W_{t-1}) \leq \frac{d}{2}\log\left( 1 + \frac{\eta_t^2L^2}{d\sigma_t^2}\right).
\end{equation*}
\label{lemma:infoBound}
\end{lemma}
\begin{proof}
Note that
\begin{align*}
I(W_t;Z_t | W_{t-1}) = h(W_t | W_{t-1}) - h(W_t | W_{t-1},Z_t).
\end{align*}
We now bound each of the terms in the final expression. First, note that conditioned on $W_{t-1} = w_{t-1}$, we have 
$$W_{t} - g(w_{t-1}) = \eta_tF(w_{t-1},Z_t) + \xi_t.$$
Note that
\begin{equation*}
h(W_t-g(w_{t-1}) \mid W_{t-1}) = h(W_t \mid W_{t-1}=w_{t-1}),
\end{equation*}
since translation does not affect the entropy of a random variable. Also note that the random variables $\xi_t$ and $\eta_tF (w_{t-1}, Z_t)$ are independent, so we can upper-bound the expected squared-norm of $W_{t} - w_{t-1}$, as follows:
\begin{align*}
\E\left(\|W_t - w_{t-1} \|_2^2\right) &= \E\left(\|\eta_tF(w_{t-1}, Z_t)\|_2^2 + \|\xi_t\|_2^2\right)\\
&\leq \eta_t^2 L^2 + d\sigma_t^2,
\end{align*}
where in the last inequality, we have used Assumption~\ref{AssBound} and the fact that $\xi_t \sim {\cal N}(0, \sigma_t^2 I_d)$. Among all random variables $X$ with a fixed $\mathbb E \| X \|_2^2 < C$, the Gaussian distribution $Y \sim {\cal N}\left(0, \sqrt{\frac{C}{d}} I_d\right)$ has the largest entropy, given by
$$ h(Y) = \frac{d}{2} \log \left(\frac{2\pi e C}{d}\right).$$
% Need to make a Lemma to prove the above. Or refer to an existing proof in a source.}
This implies that
\begin{align*}
h(W_t \mid W_{t-1} = w_{t-1}) &\leq \frac{d}{2} \log \left(2\pi e \frac{\eta_t^2L^2 + d\sigma_t^2}{d}\right).
\end{align*}
Since the above bound holds for all values $w_{t-1}$, we may integrate the bound to conclude that
\begin{align*}
h(W_t | W_{t-1}) \leq \frac{d}{2} \log \left(2\pi e \frac{\eta_t^2L^2 + d\sigma_t^2}{d}\right).
\end{align*}
We also have
\begin{align*}
h(W_t | W_{t-1}, Z_t) &= \nonumber h(W_{t-1} + \eta_t \nabla_w \ell(W_{t-1},Z_t)\\
                   &\quad + \xi_t | W_{t-1}, Z_t)\\
&= h(\xi_t | W_{t-1}, Z_t)\\
&= h(\xi_t).
\end{align*}
This leads to the following desired bound:
\begin{align*}
h(W_t | W_{t-1}) & - h(W_t | Z_t, W_{t-1}) \\
& \leq \frac{d}{2} \log \left(2\pi e \frac{\eta_t^2L^2 + d\sigma_t^2}{d}\right) - \frac{d}{2} \log 2 \pi e \sigma_t^2 \\
&= \frac{d}{2}\log \frac{\eta_t^2L^2 + d\sigma_t^2}{d\sigma_t^2} \\
&= \frac{d}{2} \log \left(1+ \frac{\eta_t^2L^2}{d\sigma_t^2}\right).
\end{align*}
Note that if the noise were non-Gaussian, we would have to replace $\frac{d}{2} \log(2 \pi e \sigma_t^2)$ by the entropy of the noise.
\end{proof}

%%%%%

\section{Details for the SGLD algorithm}
\label{AppOptSGLD}

In this Appendix, we include more details for the derivations concerning SGLD in Section~\ref{SecExamples}.

\subsection{Generalization error bounds}

%We restate Corollary 1 in Xu and Raginsky~\cite{XuRag17} (with $g(n) = \sqrt{n}$): %\textcolor{red}{Is the statement of Xu and Raginsky's corollary necessary?}

%\begin{lemma} If $\epsilon \leq (\sqrt{n}- 1)\beta \log(\frac{2}{\beta}) $, then $n = \frac{64 R^4 }{\alpha^4} (\log(\frac{2}{\beta}))^2$ satisfies inequality~\eqref{eq:pacResults}.
%\label{lemma:hpbound}
%\end{lemma}

\begin{lemma} For a given choice of $\{\beta, \alpha\}$, taking $n \geq \frac{64 R^4 }{\alpha^4} (\log(\frac{2}{\beta}))^2$ ensures inequality~\eqref{eq:pacResults}, provided that we run $ K \leq \frac{1}{ne}\left(\frac{2 }{\beta} ^ {\frac{2(\sqrt{ n} -1) \beta }{cL^2}}\right) $ epochs. 
\label{lemma:epochBound}
\end{lemma}
\begin{proof}
If we show that for $K  \leq \frac{1}{ne}\left( \frac{2 }{\beta} ^ {\frac{2(\sqrt{n}-1) \beta }{cL^2} }\right) $, we have $I(S;W) \leq (\sqrt{n}- 1)\beta \log\left(\frac{2}{\beta}\right)$, the proof will follow from Corollary~\ref{CorHP}. We have
\begin{align*}
I(S;W) &\leq \sum_{t=1}^T\frac{\eta L^2}{2} = \sum_{t=1}^T\frac{c L^2}{2t} \leq \frac{cL^2}{2}\log(eT)\\
        &= \frac{cL^2}{2}\log(enK) \\
        &\leq \frac{cL^2}{2}\log\left(\frac{2 }{\beta} ^ {\frac{2(\sqrt{ n} -1) \beta }{cL^2}) }\right)\\
        &=(\sqrt{ n} -1) \beta  \log\left(\frac{2 }{\beta}\right),
\end{align*}
implying the desired result.
\end{proof}

\subsection{Optimization error bounds}

We now derive the bound on the optimization error $\epsilon_{\text{opt}}^W$ of SGLD.

\begin{lemma}
If we run the SGLD algorithm on an $L$-Lipschitz convex loss function for $T$ time steps with parameters $\{\eta,\sigma\}$, we have the following bound on the empirical risk for the average of the iterates: 
\begin{align*}
      \E\left[ L_S\left(\frac{1}{T} \sum_{t=1}^TW_t\right)\right] - L_S(w_S^*) \leq \E\left[\frac{1}{T} \sum_{t=1}^T L_S(W_t)\right] - L_S(w_S^*) \leq \frac{G^2}{2\eta T} +  \frac{\eta}{2} L^2 + \frac{d\sigma^2}{2\eta}.
  \end{align*}
  \label{claim:ermsgld}  
  We follow the same notation as in the rest of the paper.
\end{lemma}

\begin{proof}
The first inequality follows from the convexity of the the loss fuction.

We can write the update equation as 
\begin{align*}
W_t &= W_{t-1} - \eta \left(\nabla_w \ell(W_{t-1},Z_t) + \frac{\xi}{\eta}\right)\\
&= W_{t-1} - \eta V_{t-1},
\end{align*}
where $V_t = \nabla_w \ell(W_{t-1},Z_t) + \frac{\xi}{\eta}$.

It is easy to see that $V_t$ is an unbiased estimator of the gradient of the empirical risk; i.e., $\E[V_t | W_t] = \nabla L_S(W_{t})$.
Therefore, SGLD may be seen as a variant of SGD, and we obtain the following bounds on the optimization error for the average of iterates, $W = \frac{1}{T}\sum_{t=1}^T W_t$ (cf.\ Lemma 14.1 and Theorem 14.8 of Shalev-Shwartz and Ben-David \cite{ShaBen14}):
\begin{align}
\epsilon_{\text{opt}}^W \leq \frac{G^2}{2\eta T} + \frac{\eta}{2}(\E[\|V_t \|_2^2]).
\label{EqnSgldOptError}
\end{align}
Moreover, since the noise is independent and the loss function is convex and $L$-Lipschitz, we have
\begin{equation*}
\E[\|V_t\|_2^2] \leq L^2 + d \frac{\sigma^2}{\eta^2}.
\end{equation*}
Combining this bound with inequality~\eqref{EqnSgldOptError} yields the desired result.
\end{proof}

\end{appendix}

%\section{Conclusion}
%\section{Appendices}
\end{document}